\documentclass{article}

\usepackage{hyperref}

\usepackage[accepted]{icml2020}

\usepackage[utf8]{inputenc} 
\usepackage[T1]{fontenc}    
\usepackage{url}            
\usepackage{hyperref}
\usepackage{booktabs}       
\usepackage{amsfonts}       
\usepackage{nicefrac}       
\usepackage{microtype}
\usepackage{graphicx}
\graphicspath{{./figures/}}
\usepackage{amsfonts}
\usepackage{mathtools}
\usepackage{amsmath}
\usepackage{amsthm}
\usepackage{amssymb}
\usepackage{bm}

\newcommand{\ep}{\mathbb{E}}

\newcommand{\F}{\mathcal{F}}
\newcommand{\xG}{\mathcal{G}}
\newcommand{\xD}{\mathcal{D}}
\newcommand{\uV}{\mathcal{V}}
\newcommand{\uU}{\mathcal{U}}
\newcommand{\uR}{\mathcal{R}}

\newcommand{\ty}{\bm{y}}
\newcommand{\tu}{\bm{u}}
\newcommand{\tc}{\bm{c}}
\newcommand{\tth}{\bm{h}}
\newcommand{\ttheta}{\bm{\theta}}
\newcommand{\tphi}{\bm{\phi}}

\newcommand{\fY}{\bm{Y}}
\newcommand{\fU}{\bm{U}}
\newcommand{\fC}{\bm{C}}
\newcommand{\fH}{\bm{H}}
\newcommand{\fT}{\bm{T}}
\newcommand{\fM}{\bm{M}}
\newcommand{\fTheta}{\bm{\Theta}}
\newcommand{\fPhi}{\bm{\Phi}}

\newcommand{\Real}{\mathbb{R}}
\newcommand{\Complex}{\mathbb{C}}

\newcommand{\fig}[1]{Fig.~\ref{fig:#1}}
\newcommand{\tabl}[1]{Table~\ref{tab:#1}}
\newcommand{\eqn}[1]{Eqn.~\eqref{eqn:#1}}
\newcommand{\secref}[1]{Sec.~\ref{sec:#1}}
\usepackage{wrapfig}
\usepackage{subfig}
\usepackage{algorithm}
\usepackage{algorithmic}
\usepackage{wrapfig,booktabs}

\newtheorem{assumption}{Assumption}
\newtheorem{proposition}{Proposition}

\newtheorem{lemma}{Lemma}
\newtheorem{theorem}{Theorem}
\newtheorem{prop}{Proposition}
\newtheorem{definition}{Definition}
\theoremstyle{definition}
\newtheorem{myexp}{Example}

\newcommand{\tabincell}[2]{\begin{tabular}{@{}#1@{}}#2\end{tabular}}

\usepackage{amssymb}
\usepackage{pifont}
\newcommand{\xmark}{\ding{55}}%

\begin{document}

\twocolumn[
\icmltitle{Understanding and Stabilizing GANs' Training Dynamics using Control Theory}
\icmlsetsymbol{equal}{*}
\begin{icmlauthorlist}
	\icmlauthor{Kun Xu}{thu}
	\icmlauthor{Chongxuan Li}{thu}
	\icmlauthor{Jun Zhu}{thu}
	\icmlauthor{Bo Zhang}{thu}
\end{icmlauthorlist}

\icmlaffiliation{thu}{Dept. of Comp. Sci. \& Tech., Institute for AI, BNRist Center, Tsinghua-Bosch ML Center, THBI Lab, Tsinghua University, Beijing, China}

\icmlcorrespondingauthor{Jun Zhu}{dcszj@mail.tsinghua.edu.cn}
\icmlkeywords{Machine Learning, ICML}
\vskip 0.3in
]


\printAffiliationsAndNotice{} %

\begin{abstract}
Generative adversarial networks~(GANs) are effective in generating realistic images but the training is often unstable.
There are existing efforts that model the training dynamics of GANs in the parameter space but the analysis cannot directly motivate practically effective stabilizing methods.
To this end, we present a conceptually novel perspective from control theory to directly model the dynamics of GANs in the function space and provide simple yet effective methods to stabilize GANs' training.
We first analyze the training dynamic of a prototypical Dirac GAN and adopt the widely-used closed-loop control (CLC) to improve its stability.
We then extend CLC to stabilize the training dynamic of normal GANs, where CLC is implemented as a squared $L2$ regularizer on the output of the discriminator.
Empirical results show that our method can effectively stabilize the training and obtain state-of-the-art performance on data generation tasks.
\end{abstract}

\section{Introduction}\label{sec:introduction}

Generative adversarial networks~(GANs)~\citep{goodfellow2014generative} have shown promise in generating realistic natural images~\citep{brock2018large} and facilitating unsupervised and semi-supervised learning~\citep{chen2016infogan,chongxuan2017triple,donahue2019large}.
In GANs, an implicit generator $\xG$ is defined by mapping a noise distribution to the data space. 
Since no density function is defined for the implicit generator, the maximum likelihood estimate is infeasible for GANs. Instead, a discriminator $\xD$ is introduced to estimate the density ratio between the data distribution $p$ and the generating distribution $p_\xG$ by telling the real samples from fake ones. $\xG$ aims to recover the data distribution by maximizing this ratio. This framework is formulated as a minimax optimization problem, which can be solved by optimizing $\xG$ and $\xD$ alternately.
In practice, however, GANs suffers from the instability of training~\cite{goodfellow2016nips}, where divergency and oscillations are often observed~\citep{liang2018generative,chavdarova2018sgan}.

Early methods~\citep{mao2017least,gulrajani2017improved, arjovsky2017wasserstein,du2018learning} introduce different types of divergences to improve the training process of GANs. Their theoretical analyses assume that $\xD$ achieves its optimum when training $\xG$.
However, the practical training process (e.g., alternative stochastic gradient descent) often violates the above assumption and therefore is not guaranteed to converge to the desired equilibrium. Several empirical regularizations~\cite{miyato2018spectral,gulrajani2017improved,zhang2019consistency} are used to improve the training process whereas no stability can be guaranteed.

Recently, \citet{mescheder2017numerics} and \citet{nagarajan2017gradient} directly model the training dynamics of GANs, i.e. how the parameters develop over time. 
Formally, the dynamic is defined as the gradient flow of the parameters.
The stability of the dynamic is fully determined by the eigenvalues of the Jacobian matrix of the gradient flow.
Indeed, the stability analysis in a linear prototypical GAN (i.e. Dirac GAN~\cite{mescheder2018training}) is elegant. 
However, this analysis does not directly motivate effective algorithms to stabilize GANs' training.
To our knowledge, such methods do not report competitive image generation results to the state-of-the-art GANs~\cite{miyato2018spectral}.

In this paper, we understand and stabilize GANs' training dynamics from the perspective of control theory.
Based on the recipe for control theory, we can not only analyze the dynamics of Dirac GAN formally, but also develop practically effective stabilizing methods for nonlinear dynamics~\cite{khalil2002nonlinear}.
Specifically, we start from revisiting the Dirac GAN example with the WGAN's objective function in \secref{dirac_gan}. By utilizing the Laplace transform~\cite{widder2015laplace}~(LT), the training dynamics of both $\xD$ and $\xG$ can be modeled in the {\it frequency domain} instead of the {\it time domain} in previous methods~\cite{mescheder2017numerics, mescheder2018training}.
These types of dynamics are well studied in control theory and the stability can be easily inferred. The analysis can be simply generalized to other objective functions with {\it local linearization}. Given the instability of GANs, the recipe for control theory provides a set of tools to stabilize their dynamics.
We first adopt the {\it closed-loop control}~(CLC) to successfully stabilize the dynamic of Dirac GAN with theoretical guarantee.
Besides, extensive empirical results in control theory show that the CLC is also helpful in nonlinear settings~\cite{khalil2002nonlinear}.
It inspires us to extend our proposal to normal GANs by modeling $\xD$ and $\xG$'s dynamics in the function space where these dynamics and Dirac GAN's dynamics share similar forms and characters.
The CLC is implemented as a regularization term to $\xD$'s objective function which penalizes the squared $L2$ norm of the output of $\xD$ as we described in \secref{implementation}.
We therefore refer our method as CLC-GAN.
CLC-GAN is verified on an 1-dimension toy example as well as the natural images including CIFAR10~\citep{krizhevsky2009learning} and CelebA~\citep{liu2015faceattributes}. The results demonstrate that our method can successfully stabilize the dynamics of GANs and achieve state-of-the-art performance.

Our contributions are summarized as:
\begin{itemize}
	\setlength\itemsep{-2pt}
	\item We formally analyze the training dynamics of GANs from a novel perspective of control theory, which is generally applicable to different objective functions.
	\item We propose to use the CLC as an effective method to stabilize the training of GANs, while other advanced control methods can be explored in future. 
	\item The simulated results on Dirac GAN agree with the theoretical analysis and CLC-GAN achieves the state-of-the-art performance on natural image generations.
\end{itemize}

\section{Preliminary}
\label{sec:preliminary}

In this section, we present the recipe for control theory, especially under the Laplace transform, which is powerful to model dynamic systems and design stabilizing methods.

\subsection{Modeling Dynamic Systems }
\label{sec:dynamicmodeling}

In control theory, a {\it signal} is represented as a function over time $t$, i.e., in the {\it time domain}~\cite{kailath1980linear}. 
A dynamic\footnote{For simplicity, we use {\it dynamic} for dynamic system.} represents how one signal (i.e., output, denoted by $\ty(t)$) develops with respect to another signal (i.e., input, denoted by $\tu(t)$) over time. A natural representation of a dynamic is a differential equation~(DE)\footnote{We consider ordinary differential equations in this paper.}:
\begin{align}
\frac{d\ty(t)}{dt} = f(\ty(t), \tu(t)),
\label{eqn:general_dynamic}
\end{align}
together with an initial condition $\ty(0) = \ty_0$.
Note that $f(\cdot, \cdot)$, $\ty(t)$ and $\tu(t)$ can be vector valued functions. 
We assume $\ty_0 = 0$ unless specified.
A dynamic is {\it linear} if $f(\cdot, \cdot)$ is a linear function.

Besides the time domain, a signal can also be represented as a function of frequency $s$, i.e., in the {\it frequency domain}.
A DE of a linear dynamic in the time domain can be converted to a simple algebraic equation in the frequency domain, which can largely simplify the solving process and stability analysis of a dynamic. Laplace transform~\cite{widder2015laplace}~(LT) is a widely-adopted operator to convert signals from the time domain to the frequency domain. Formally, LT is given by:
\begin{align}
\F(\tth)(s)=\int_{0}^{\infty} \tth(t) e^{-st} dt = \fH(s), \label{eqn:lap_definition}
\end{align}
where $\tth$ is a signal in the time domain, and $s=\sigma + \omega i\in \Complex$ with real numbers $\sigma$ and $\omega$. The real and imaginary parts of $\fH(s)\in\Complex$ denote the gain and phase of the frequency $s$ in $\tth$. In this paper, we use bold lowercase letters (e.g., $\ty, \tu$) to denote signals in the time domain and bold capital letters (e.g. $\fY, \fU$) to denote signals in the frequency domain.

Leveraging LT, the derivation over time $t$ can be represented as multiplying a factor $s$ in the frequency domain:
\begin{align}
\F(\frac{d\tth(t)}{dt}) = s \F(\tth).
\end{align}
Therefore, by applying LT to both sides of a DE in \eqn{general_dynamic}, a linear dynamic can be solved by the formal rules of algebra and represented in the form of $\fY(s) = \fT(s)\fU(s)$, where $\fT(s)$ is a simple rational fraction called {\it transfer function}~\cite{kailath1980linear}.
The transfer function can facilitate the stability analysis, as detailed in \secref{stability_analysis}.

\subsection{Stability Analysis}
\label{sec:stability_analysis}

In general, we require a dynamic to be {\it stable}. Although different definitions exist, we consider the widely adopted asymptotic stability\footnote{This definition is consistent with existing work in \citet{mescheder2017numerics} and \citet{mescheder2018training}.}~\cite{kailath1980linear} in this paper.
\begin{definition}
	For a constant input $\tu(t) = \tu_c$, a point $\ty_e$ is called an equilibrium point of a dynamic represented in \eqn{general_dynamic}, if $f(\ty_e, \tu_c) = 0$. A dynamic is called asymptotically stable if for every $\epsilon$ > 0, there exists $\sigma > 0$ such that if $||\ty(0) - \ty_e|| < \sigma$, then for every $t>0$, $||\ty(t) - \ty_e|| < \epsilon$ and $\lim_{t\to \infty} ||\ty(t) - \ty_e|| = 0$. Here $||\cdot||$ is a norm defined in the vector space of $\ty$.
\end{definition}

In the frequency domain, the stability can be directly inferred from the transfer function. Formally, we define {\it poles} as the roots of the denominator in a transfer function. The stability of a linear dynamic is fully determined by its poles as summarized in the following proposition.\newpage

\begin{prop}(Theorm 2.6-1 in \citet{kailath1980linear})
\vspace{-.15cm}
\begin{enumerate}
	\setlength\itemsep{-2pt}
	\item A dynamic is asymptotic stable if all poles have negative real parts.
	\item A dynamic is oscillatory (i.e., bounded output but not stable) if one or more poles are purely imaginary.
	\item A dynamic is diverged (unbounded output) if one or more poles have positive real parts.
\end{enumerate}
\label{prop_stability}
\end{prop}

\subsection{Control Methods}
\label{sec:control_method}

For an unstable dynamic, control theory provides a set of methods to improve its stability. Among them, 
the {\it closed-loop control}~\cite{kailath1980linear}~(CLC) is one of the most popular ones and robust to nonlinearity in dynamics practically. 

The central idea is to modify the transfer function by feeding the output back to the input such that all poles have negative real parts.
Specifically, we introduce an additional dynamics called {\it controllers} with transfer functions $\fT_b(s)$ to adjust the output signal and input signal respectively. The controller takes $\fY(s)$ as input and output the feedback signal $\fY_b = \fT_b(s)\fY(s)$. We then substitute the difference between $\fU$ and $\fY_b$ (i.e., $\fM = \fU-\fY_b$) for input in the original dynamics, resulting the output signal as $\fY(s) = \fT(s)\fM(s)$. The relationship between the input $\fU(s)$ and the output $\fY(s)$ is:
\begin{align}
\fY(s) = \fT(s)(\fU(s) - \fT_b(s)Y(s)).\label{eqn:feed_back_negative}
\end{align}
Further, the whole controlled dynamic is given as:
\begin{align}
\fY(s)=\frac{\fT(s)}{1+\fT_b(s)\fT(s)}\fU(s).\label{eqn:feedback_transfer}
\end{align}
With a properly designed $\fT_b$, the poles of the dynamic in \eqn{feedback_transfer} can have negative real parts and the dynamic is stabilized. In the following, we first model and stabilize the training dynamic of Dirac GAN: a simplified GAN with linear dynamics in \secref{dirac_gan} and then we generalize it to the realistic setting in \secref{normal_gan}.

\section{Analyzing Dirac GAN by Control Theory}\label{sec:dirac_gan}

In this section, we focus on the Dirac GAN~\cite{mescheder2018training}, which is a widely adopted example to analyze the stability of GANs.
Previous work~\cite{mescheder2017numerics,gidel2018negative} uses the Jacobian matrix to analyze the stability of dynamics whereas does not directly provide an approach to stabilize it. Instead, we revisit this example from the perspective of control theory and develop a principled method that not only analyzes but also improves the stability of various GANs.


\subsection{Modeling Dynamics}
\label{sec:gan_dynamic_model}


We first model the dynamics of the Dirac GANs in the language of control theory, which can facilitate the stability analysis and improvement in \secref{ana_imp_stability}.
In Dirac GAN, $\xG$ is defined as $p_\xG(x) = \delta(x-\theta)$  where $\delta(\cdot)$ is the Dirac delta function, and $\xD$ is defined as $\xD(x) = \phi x$. $\theta$ and  $\phi$ are the parameters of $\xG$ and $\xD$ respectively. The data distribution is $p(x) = \delta(x - c)$ with a constant $c$. 
Generally, the objective functions of $\xD$ and $\xG$ can be written as:
\begin{align}
& \max_{\phi} \uV_1(\phi; \theta) = h_1(\xD(c)) + h_2(\xD(\theta)), \nonumber\\
& \max_{\theta} \uV_2(\theta; \phi) = h_3(\xD(\theta)).
\end{align}
Here $h_i(\cdot): \Real\to\Real$ is a scalar function for $i\in\{1,2,3\}$.  
Assuming that the equilibrium point of $\xD$ is a zero function as in most GANs~\cite{goodfellow2014generative,arjovsky2017wasserstein}, it is required that $h_1(\cdot)$ and $h_3(\cdot)$ are increasing functions and $h_2(\cdot)$ is a decreasing function around zero. For instance, when $h_1(x) = h_3(x) = \log(\sigma(x))$ and $h_2(x) = \log(1-\sigma(x))$ with $\sigma(\cdot)$ denoting the sigmoid function, we obtain the vanilla GAN~\cite{goodfellow2014generative}.

Since $\theta$ and $\phi$ are updated using gradient descent, we can denote the training trajectories as signals $\ttheta$ and $\tphi$. The dynamics are defined by the following gradient flow:
\begin{align}
\frac{d\tphi(t)}{dt} = \frac{\partial \uV_1(\phi; \theta)}{\partial \phi}|_{\phi=\tphi(t),\theta=\ttheta(t)}, \nonumber\\
\frac{d\ttheta(t)}{dt} = \frac{\partial \uV_2(\theta; \phi)}{\partial \theta}|_{\phi=\tphi(t),\theta=\ttheta(t)}.
\end{align}
Specifically, for the dynamics of $\xD$, we have:
\begin{align}
\frac{\partial \uV_1(\phi; \theta)}{\partial \phi} = &\frac{d h_1(\xD(c))}{d\phi}+ \frac{d h_2(\xD(\theta)) }{d\phi} \label{eqn:dynamic_dirac_D}.
\end{align}
Similarly, for the dynamics of $\xG$, we have:
\begin{align}
\frac{\partial \uV_2(\theta;\phi)}{\partial \theta} &= \frac{d h_3(\xD(\theta))}{d \xD(\theta)}\frac{\partial \xD(\theta)}{\partial \theta} \label{eqn:dynamic_dirac_G}.
\end{align}
Substituting $\xD(x) = \phi x$ to \eqn{dynamic_dirac_D} and \eqn{dynamic_dirac_G}, the dynamics of Dirac GAN can be summarized as:
\begin{align}
& \frac{d\tphi(t)}{dt} = h_1^\prime(\tphi(t) c)c + h_2^\prime(\tphi(t) \ttheta(t))\ttheta(t), \nonumber\\
& \frac{d\ttheta(t)}{dt} = h_3^\prime(\tphi(t) \ttheta(t))\tphi(t), \label{eqn:dynamic_dirac_summarize}
\end{align}
where $h_i^\prime(\cdot)$ denotes the derivative of $h_i(\cdot)$ for $i\in\{1, 2, 3\}$.

From the perspective of control theory (see details in \secref{dynamicmodeling}), \eqn{dynamic_dirac_summarize} represents a dynamic in the time domain, which is natural to understand but difficult to analyze. Converting it to the frequency domain by the Laplace transform~(LT) can simplify the analysis. It requires a case by case derivation for different GANs due to the specific forms of the objective functions~(i.e., different choices of $h_i(\cdot)$).
We will first use WGAN as an example to present the analyzing process and then generalize it to other objectives via the local linearization technique in \secref{local_linearize}.

In WGAN\footnote{We ignore the Lipschitz continuity of $\xD$ for simplicity but the equilibrium point and its local convergence do not change. See theoretical analysis and empirical evidence in Appendix~B.}, we have $h_1(x) = h_3(x) = x$ and $h_2(x) = -x$. Let the output $\ty(t)=(\ttheta(t), \tphi(t))$ and the input $\tu(t) = c,~\forall t > 0$. Then, the dynamic in \eqn{dynamic_dirac_D} and \eqn{dynamic_dirac_G} is instantiated as:
\begin{align}
\frac{d\ty(t)}{dt} = \begin{bmatrix}0 & 1 \\ -1 & 0 \end{bmatrix}\begin{bmatrix}\ttheta(t) \\ \tphi(t) \end{bmatrix} + \begin{bmatrix}0 \\ \tu(t)\end{bmatrix} = f(\ty(t), \tu(t)). \label{eqn:dirac_gan}
\end{align}
Applying LT $\mathcal{F}(\cdot)$ in \eqn{lap_definition} to both sides of \eqn{dirac_gan}, the dynamic can be represented in the frequency domain as:
\begin{align}
\begin{cases}
s\fPhi(s) = \fU(s) - \fTheta(s),\\
s\fTheta(s) = \fPhi(s).
\end{cases}
\label{eqn:dynamic_diracgan_lap}
\end{align}
where $\fTheta, \fPhi, \fU$ represent $\ttheta, \tphi, \tu$ in the frequency domain, e.g., $\fU(s) = \mathcal{F}(\tu)(s)$.
Then we can solve the dynamics of $\fPhi$ and $\fTheta$ according to the formal rules of algebra as:
\begin{align}
\begin{cases}
\fPhi(s) = \frac{s}{s^2 + 1}\fU(s), \\
\fTheta(s) = \frac{1}{s}\fPhi(s) = \frac{1}{s^2 + 1}\fU(s).
\end{cases}\label{eqn:dynamic_frequency_parameter}
\end{align}
In the frequency domain, the output signal can be represented as a multiplication between the transfer function (see \secref{dynamicmodeling}) and the input signal.
Specifically, in \eqn{dynamic_frequency_parameter}, the transfer function of $\tphi$ is $\fT_\xD(s) = \frac{s}{s^2 + 1}$ and the transfer function of $\ttheta$ is $\fT_\xG(s) = \frac{1}{s^2+1}$.
According to Proposition \ref{prop_stability}, the stability of a dynamic is fully characterized by the poles of the transfer function (i.e., the roots of the denominator).
The poles of both $\ttheta$ and $\tphi$ are $\pm i$ according to \eqn{dynamic_frequency_parameter}. Therefore, both $\ttheta$ and $\tphi$ are oscillatory instead of converging to the equilibrium point $(\ttheta_e, \tphi_e) = (c, 0)$. The simulated dynamic of Dirac GAN is illustrated in \fig{simulated_diracgan}. 


\subsection{Analyzing and Improving Stability}
\label{sec:ana_imp_stability}

\begin{figure}
	\centering
	\includegraphics[width=0.22\textwidth]{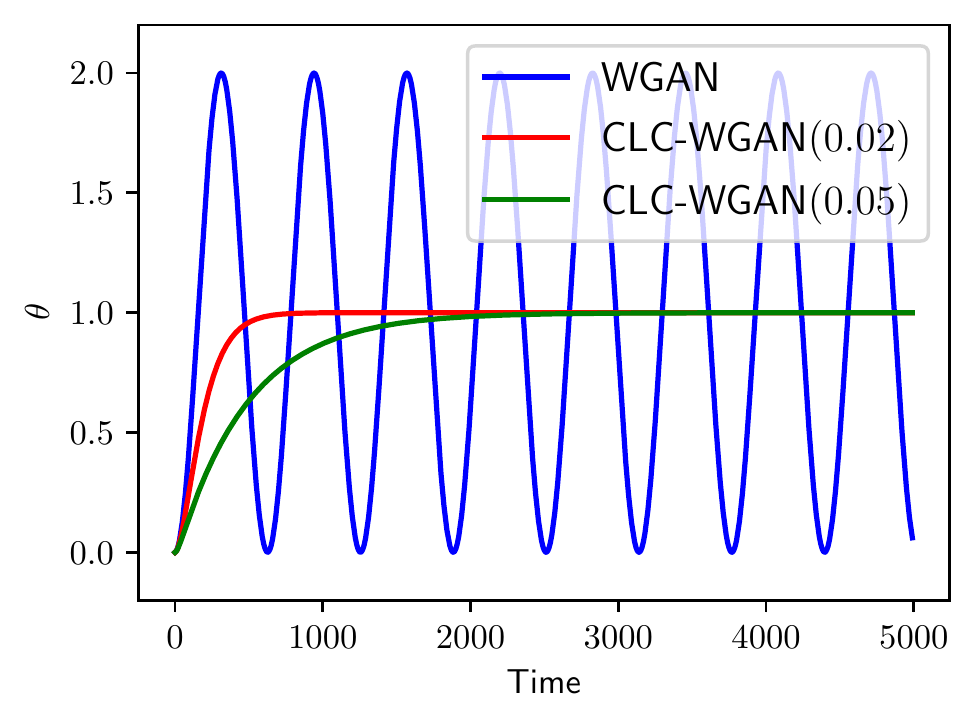}
	\includegraphics[width=0.22\textwidth]{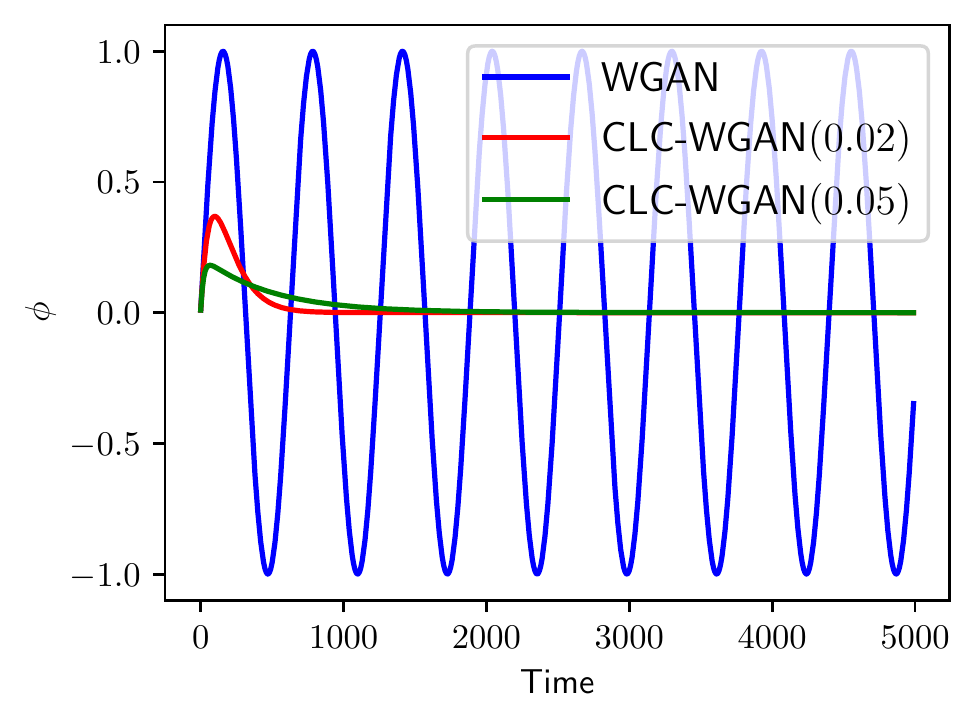}
	\caption{The simulated dynamic of Dirac GAN for $\ttheta$ (left) and $\tphi$ (right) with $c=1$. The curve of WGAN shows the oscillation while Other curves of CLC-GAN show that the closed loop control helps convergence.}
	\label{fig:simulated_diracgan}
\end{figure}

Control theory provides extensive methods~\cite{khalil2002nonlinear} to improve the stability of dynamics without changing the desired equilibrium. In this paper, the widely used closed-loop control~(CLC) is introduced in \secref{control_method} for its simplicity. We emphasize that advanced control methods can potentially result in more stable GANs and we leave it as future work.

Before applying the CLC, we emphasize that there are two requirements to be satisfied simultaneously: 1) applying the CLC needs to stabilize the dynamics of $\xD$ and $\xG$; 2) it should not change the equilibrium point of $\xG$, i.e., $p_\xG = p$. 

For the first requirement, the dynamic of $\ttheta$ in Dirac GAN is $\frac{d\ttheta(t)}{dt} = h_3^\prime(\tphi(t) \ttheta(t))\tphi(t)$, which indicates that stabilizing $\tphi$ to zero can also stabilize the dynamic of $\ttheta$. Therefore, we only need to introduce the CLC to $\xD$.
The central idea of the CLC is to adjust the transfer function by introducing an auxiliary controller. Here we adopt a simple and widely used controller $\fT_b(s)=\lambda$. Intuitively, it is an amplifier with negative feedback from output to input according to \eqn{feed_back_negative} and $\lambda$\footnote{$\lambda$ is a hyperparameter and we analyze its sensitivity in \secref{experiment}.} is the coefficient for the amplitude of the feedback.
Substituting $\fT_b$ with $\lambda$ in \eqn{feedback_transfer}, the transfer function $\fT_{c\xD}$ of the controlled $\tphi$ is given by:
\begin{align}
\fT_{c\xD}(s) = \frac{\frac{s}{s^2+1}}{1 + \frac{\lambda s}{s^2 + 1}} = \frac{s}{s^2 + \lambda s + 1}.
\end{align}
With a positive $\lambda$, all of poles in the controlled dynamic have negative real parts, and hence it is a stable dynamic. We also demonstrate the simulated results of the controlled dynamic with different values of $\lambda$ in \fig{simulated_diracgan}.

For the second requirement, the CLC will not change the equilibrium point of Dirac GAN. In the time domain, the CLC is equivalent to adjust the dynamics of $\tphi$ as:
\begin{align}
\frac{d \tphi}{dt} = c - \ttheta(t) - \lambda \tphi(t).
\end{align}
Since the equilibrium point of $\xD$ is a zero function, i.e., $\tphi_e = 0$, then we still have $\frac{d \ty(t)}{dt} = 0$ at $\ty = (\ttheta_e, \tphi_e)$.

\setlength{\tabcolsep}{1.4mm}{
	\begin{table*}[htbp]
		\centering
		\caption{The stability characters for the widely-used GANs. Please refer to Appendix A for detailed derivation, which adopts the local linearization technique introduced in \secref{local_linearize}.
			With CLC, the training dynamics of Dirac GANs are stable theoretically (see \fig{simulated_diracgan} and Appendix~A), and those of normal GANs are stable empirically (see \fig{learning_curve}).}
		\begin{tabular}{ccccc}
			\toprule
			& $\fT_\xD(s)$  & \tabincell{c}{Stability \\ Dirac GAN/normal GAN}& $\fT_{c\xD}(s)$ & \tabincell{c}{Stability with CLC \\ Dirac GAN/normal GAN}\\
			\midrule
			WGAN  & $s/(s^2+1)$ & \xmark/\xmark &  $1/(s^2+\lambda s + 1)$ & $\checkmark$/$\checkmark$ \\
			Hinge-GAN & $s/(s^2+1)$ & \xmark/\xmark & $1/(s^2+\lambda s + 1)$ & $\checkmark$/$\checkmark$ \\
			SGAN  & $2s/(4s^2+2s+1)$ & $\checkmark$/\xmark & $1/(4s^2+(2\lambda + 2)s+1)$ & $\checkmark$/$\checkmark$ \\
			LSGAN & $s/(s^2+4s+1)$ & $\checkmark$/\xmark & $1/(s^2+(\lambda + 4)s+1)$ & $\checkmark$/$\checkmark$ \\
			\bottomrule
		\end{tabular}%
		\label{tab:summarize_dynamic}%
\end{table*}}

\subsection{Extending to Other Objectives}
\label{sec:local_linearize}

The proposed method is not limited to WGAN but can be generalized to other GANs~\cite{goodfellow2014generative, mao2017least}, which may have nonlinear objective functions.

We leverage a standard technique called {\it local linearization}~\cite{khalil2002nonlinear} to approximate the original dynamics as a linear one around the equilibrium point.
For example, the objective function of $\xD$ in the vanilla GAN is:
\begin{align}
\max_{\phi} \uV_s(\phi, \theta) = \log(\sigma(\phi c)) + \log(1-\sigma(\phi \theta)).
\end{align}
The dynamic of $\tphi$ is nonlinear because of the sigmoid function, which is given by:
\begin{align}
&\frac{d\tphi(t)}{dt} = \frac{\partial \uV_s(\phi, \theta)}{\partial \phi}|_{\phi=\tphi(t), \theta=\ttheta(t)} \nonumber\\
&= \frac{\sigma^\prime(\tphi(t) c)}{\sigma(\tphi(t) c)}c - \frac{\sigma^\prime(\tphi(t) \ttheta(t))}{1 - \sigma(\tphi(t) \ttheta(t))}\ttheta(t),
\end{align}
where $\sigma^\prime(\cdot)$ is the derivative of $\sigma(\cdot)$.
Local linearization approximates the original dynamic by the first order Taylor expansion at the equilibrium point $(c, 0)$:
\begin{align}
&\frac{\partial \uV_s(\phi, \theta)}{\partial \phi}\approx \frac{\partial \uV_s(\phi, \theta)}{\partial \phi}|_{\phi=0, \theta=c} + 
\frac{\partial^2 \uV_s(\phi, \theta)}{\partial \phi^2}|_{\phi=0, \theta=c}\phi  \nonumber\\
&+ \frac{\partial^2 \uV_s(\phi, \theta)}{\partial \theta\partial \phi}|_{\phi=0, \theta=c}(\theta-c) 
= -\frac{1}{2} \phi - \frac{1}{2}(\theta - c).
\label{eqn:local_linear}
\end{align}

Note that the stability is determined by the local character of the equilibrium point, around which the residual in \eqn{local_linear} is negligible. Therefore, we have a linear approximation and the the analysis in \secref{ana_imp_stability} applies. We summarize the stability characters for all GANs in \tabl{summarize_dynamic}.

\begin{algorithm}[t]
	\begin{algorithmic}[1]
		\caption{Cloosed-loop Control GAN}\label{algo:NFGAN}
		\STATE {\bfseries Input:} Buffer size $N_b$, feedback coefficient $\lambda$, batch size $N$, initialized $\xG$ and $\xD$, learning rate $\eta$.
		\STATE Initialize $B_r$ and $B_f$ for real samples and fake samples respectively.
		\REPEAT
		\STATE Sample a batch of $\{x_r\} \sim p$, $\{x_f\}\sim p_\xG$ of $N$ samples.
		\STATE Update $B_r$ with $\{x_r\}$. Update $B_f$ with $\{x_f\}$.
		\STATE Sample a batch of $x_r^\prime \sim B_r$, $x_f^\prime\sim B_f$ of $N$ samples respectively.
		\STATE Estimate the objective of $\xD$:\\
		$\uU(\xD)=\frac{1}{N}[\sum_{x\in\{x_r\}} \xD(x) - \sum_{x\in\{x_f\}} \xD(x)] - \frac{\lambda}{N}[\sum_{x\in\{x^\prime_r\}} \xD^2(x) + \sum_{x\in\{x^\prime_f\}} \xD^2(x)].$
		\STATE Update $\xD$ to maximize $\uU(\xD)$ with learning rate $\eta$.
		\STATE Estimate the objective of $\xG$:
		$\uU(\xG) = \frac{1}{N}\sum_{x\in\{x_f\}} \xD(x)$.
		\STATE Update $\xG$ to maximize $\uU(\xG)$ with learning rate $\eta$.
		\UNTIL{Convergence}
	\end{algorithmic}
\end{algorithm}

\section{Extensions to Normal GANs}\label{sec:normal_gan}

In \secref{dirac_gan}, we show that the dynamic of Dirac GAN can be formally analyzed and stabilized based on the recipe for control theory.
Besides, the CLC can successfully stabilize nonlinear dynamics in control theory~\cite{khalil2002nonlinear}.
This two facts inspire us to stabilize the training dynamic of a normal GAN (i.e., parameterized by neural networks) by incorporating the CLC.
Unlike previous methods~\cite{mescheder2018training} which mainly focus on the dynamics of parameters of $\xD$ and $\xG$, we instead model the dynamics of $\xG$ and $\xD$ in the function space, i.e., $\xD = \xD(x, t)$ and $\xG = \xG(z, t)$. It can simplify the analysis and build the connections between the Dirac GAN and the normal GANs.

Following the notation in \secref{dirac_gan}, the objective function of a general GAN is:
\begin{align}
& \max_{\xD} \uV_1(\xD; \xG) = \ep_{p(x)}[h_1(\xD(x))] + \ep_{p_\xG(x)}[h_2(\xD(x))], \nonumber\\
& \max_{\xG} \uV_2(\xG; \xD) = \ep_{p_z(z)} [h_3(\xD(\xG(z)))].
\end{align}
According to the calculus of variations~\citep{gelfand2000calculus}, the gradient of $\uV_1(\xD)$ with respect to the function $\xD$ is:
\begin{align}
\frac{\partial \uV_1(\xD; \xG)}{\partial \xD} = p\frac{d h_1(\xD)}{d\xD}+ p_\xG\frac{d h_2(\xD) }{d\xD},
\end{align}
where $\frac{dh_i(\xD)}{d\xD}(x) = \frac{dh_i(u)}{du}|_{u=\xD(x)} = \frac{d h_i(\xD(x))}{d \xD(x)}$ for $i \in \{1, 2\}$.
The gradient of $\uV_2(\xG)$ with respect to $\xG$ is:
\begin{align}
\frac{\partial \uV_2(\xG)}{\partial \xG} = p_z \frac{d h_3(\xD(\xG))}{d\xG},
\end{align}
where $\frac{d h_3(\xD(\xG))}{d\xG}(z) = \frac{d h_3(\xD(\xG(z)))}{d \xD(\xG(z))}\frac{\partial \xD(\xG(z))}{\partial \xG(z)}$.

Therefore, the dynamics of $\xD$ and $\xG$ in normal GANs can be denoted generally as:
\begin{align}
&\frac{d\xD(x, t)}{dt} = p(x)\frac{d h_1(\xD(x))}{d\xD(x,t)}+ p_\xG(x)\frac{d h_2(\xD(x)) }{d\xD(x)}, \forall x,\nonumber \\
&\frac{d\xG(z, t)}{dt} = p_z(z)\frac{d h_3(\xD(\xG(z)))}{d \xD(\xG(z))}\frac{\partial \xD(\xG(z))}{\partial \xG(z)}, \forall z.
\end{align}
Note that the above dynamics is quiet similar to the dynamic of Dirac GAN by substituting $\xG$ and $\xD$ for $\theta$ and $\phi$ in \eqn{dynamic_dirac_D} and \eqn{dynamic_dirac_G} respectively.
Specifically, in both dynamics, the discriminators take the weighted summation of $p$ and $p_\xG$. For the generator, both of them depend on the $\frac{\partial \xD(\xG(z))}{\partial \xG(z)}$.
The above similarity between Dirac GANs and normal GANs inspires us to directly apply the CLC in nonlinear settings. Our empirical results in various settings (see \secref{experiment}) demonstrate the effectiveness of the proposed method, which agrees with the above analysis and \tabl{summarize_dynamic}.

\subsection{Implementing CLC in GANs}
\label{sec:implementation}

According to \secref{ana_imp_stability}, we apply the CLC with a controller $\fT_b(s)=\lambda$ to normal GANs. 
The resulting dynamic of $\xD$ is
\begin{align}
\frac{d\xD(x, t)}{dt} = \frac{\partial \uV_1(\xD; \xG)}{\partial \xD} - \lambda \xD(x), \forall x.\label{eqn:dynamic_con_D}
\end{align}
Note that $\xD$ will be optimized by gradient descent in the implementation and we need to design a proper objective function whose gradient flow is equivalent to \eqn{dynamic_con_D}.
Therefore, we introduce an auxiliary regularization term to the original GANs and get:
\begin{align}
\uV_1^\prime(\xD; \xG) = \uV_1(\xD; \xG) - \frac{\lambda}{2}\int_{x\in \mathcal{X}} \xD^2(x)dx,
\end{align}
where $\mathcal{X}$ denotes the space of $x$, e.g., $\mathcal{X} = [-1, 1]^{c\times w\times h}$ for image generation of size $w\times h \times c$.
Below, we denote $\uR(\xD) = \int_{x\in\mathcal{X}} \xD^2(x)dx$, which is the squared 2-norm of the function $\xD$ over the space of $x$. Intuitively, minimizing $\uR(\xD)$ encourages $\xD$ to converge to a zero function.

The regularization term $\uR(\xD)$ is proportional to the expectation of $\xD^2$ with respect to a uniform distribution $p_u(x)$ defined on $\mathcal{X}$, i.e., $\uR(\xD) \propto \ep_{p_u(x)}[\xD^2(x)]$.
However, directly estimating $\uR(\xD)$ is not sample efficient since most of samples in $\mathcal{X}$ is meaningless and do not provide useful training signals to stabilize $\xD$.
Instead, we maintain two buffers $B_r$ and $B_f$ of fix size $N_b$ to store the old real samples and fake samples, respectively. We define a uniform distribution $p_u^t(x)$ on $B^t=B_r^t\cup B_f^t$ to approximate $\uR(\xD)$ as:
\begin{align}
\uR_t(\xD) &= \int_{x\in\mathcal{X}} p_u^t(x)\xD^2(x)dx. \label{eqn:clc_regularization}
\end{align}
where $\uR_t(\xD)$ denotes the regularization term at time $t$.
$\uR_t(\xD)$ is estimated using Monte Carlo and these buffers are updated with replacement. As analyzed below, using $\uR_t(\xD)$ to approximate $\uR(\xD)$ will not change the equilibrium and stability.
The training procedure is presented in Alg.~\ref{algo:NFGAN}.

\begin{figure}[t]
	\centering
	\includegraphics[width=0.48\textwidth]{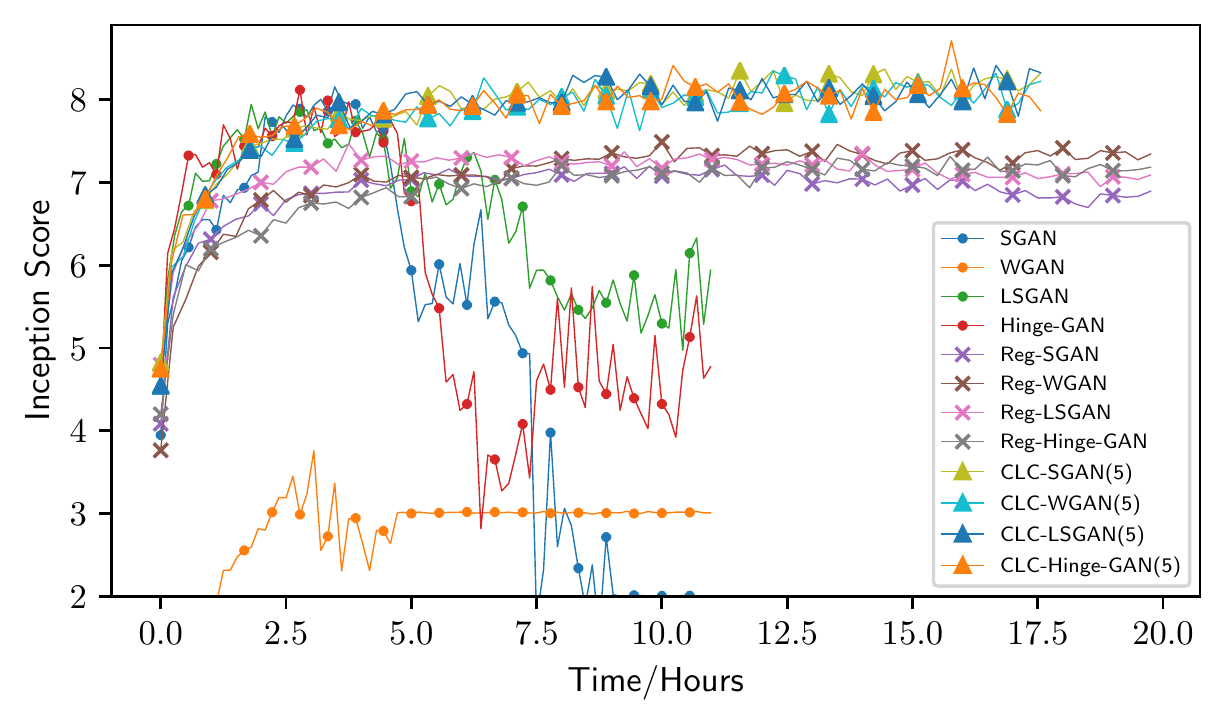}
	\includegraphics[width=0.48\textwidth]{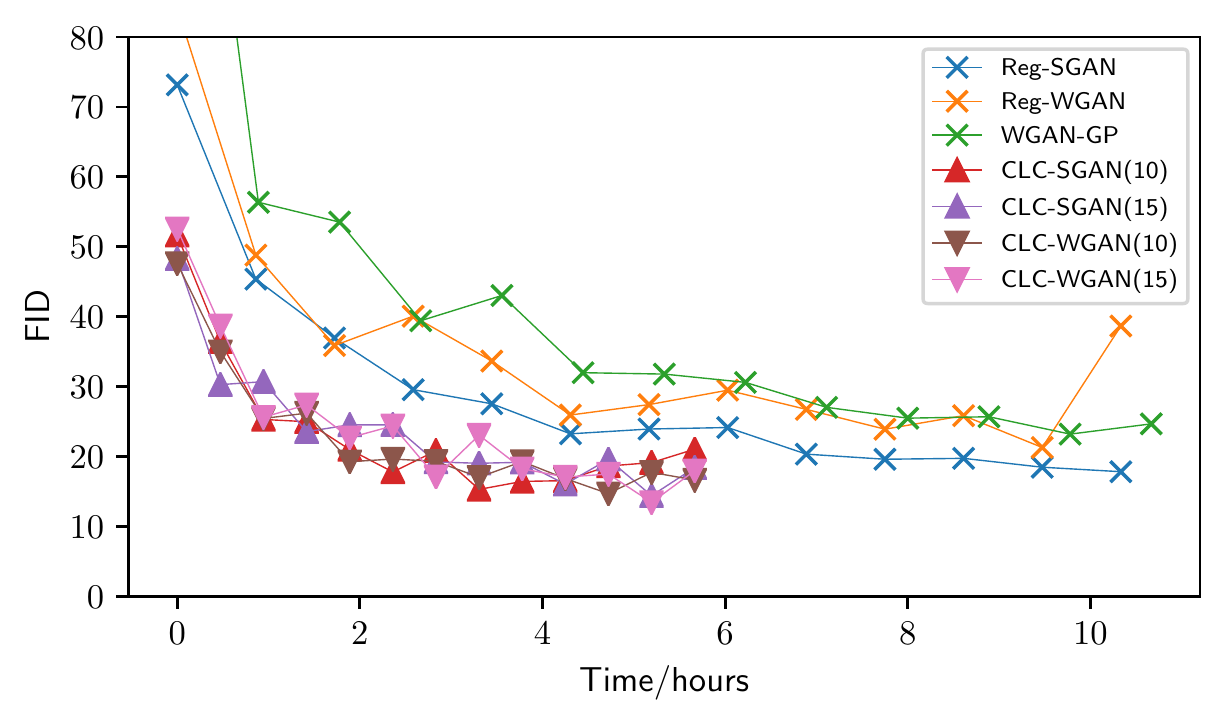}
	\caption{The learning curve of the baselines and our proposed method. Top: The Inception Score of CIFAR10. Bottom: The FID score of CelebA. We plot the curves with respect to the time for better representation of the computational cost.}
	\label{fig:learning_curve}
\end{figure}

\subsection{Theoretical Analysis}
Below, we first prove that the regularization term in \eqn{clc_regularization} will not change the desirable equilibrium point of GANs, i.e., $p_\xG=p$, as summarized in Lemma~1.

\begin{lemma}
Under the non-parametric setting, CLC-GAN has the same equilibrium as the original GAN, i.e, $p_\xG=p$ and $\xD(x)=0$ for all x.
\end{lemma}

Here we follow the identical assumption as in~\citet{goodfellow2014generative}.
Further, under mild assumptions as in~\citet{mescheder2018training}, CLC-GAN locally converges to the equilibrium, as summarized in Theorem~1.

\begin{theorem}(Proof in Appendix C)
Under the Assumptions 1, 2 and 3 in Appendix C with sufficient small learning rate and large $\lambda$, the parameters of CLC-GAN locally converge to the equilibrium with alternative gradient descent.
\end{theorem}

We provide the experimental results in \secref{experiment} to empirical validate our method.

\section{Related Work}\label{sec:related_work}


Some recent work directly models the training process of GANs.
\citet{mescheder2017numerics} and \citet{nagarajan2017gradient} model the dynamics of GANs in the parameter space and stabilize the training dynamics using gradient-based regularization. However, the above methods do not model the whole training dynamics explicitly and cannot generalize to natural images.
Then \citet{mescheder2018training} propose a prototypical example Dirac GAN to understand GANs' training dynamics and stabilize GANs using simplified gradient penalties. 
\citet{gidel2018negative} analyze the effect of momentum based on the Dirac GAN and propose the negative momentum.
Though the above methods provide an elegant understanding of the training dynamics, this understanding does not provide a practically effective algorithm to stabilize nonlinear GANs' training and they fail to report competitive results to the state-of-the-art~(SOTA) methods~\cite{miyato2018spectral}. 
Instead, we revisits the Dirac GAN from the perspective of control theory, which provides a set of tools and extensive experience to stabilize it.
Based on the recipe, we advance the previous SOTA results on image generation.

\citet{feizi2017understanding} is another related work that analyzes the stability of GANs using the Lyapunov function, which is a general approach in control theory. However, it only focuses on the stability analysis whereas cannot provide stabilizing methods. In our paper, we are interested in building SOTA GANs in practise and therefore we leverage the classical control theory.


\section{Experiments}\label{sec:experiment}

We now empirically verify our method on the widely-adopted CIFAR10~\citep{krizhevsky2009learning} and CelebA~\citep{liu2015faceattributes} datasets. 
CIFAR10 consists of 50,000 natural images of size $32\times 32$ and CelebA consists of 202,599 face images of size $64\times 64$.
The quantitative results are from the corresponding papers or reproduced on the official code for fair comparison.
Specifically, we use the exactly same architectures for both $\xD$ and $\xG$ with our baseline methods, where the ResNet~\citep{he2016deep} with the ReLU activation~\citep{glorot2011deep} is adopted\footnote{Our code is provided \href{https://github.com/taufikxu/GAN_PID}{HERE}.}.
The batch size is 64, and the buffer size $N_b$ is set to be 100 times of the batch size for all settings. 
We manually select the coefficient $\lambda$ among $\{1, 2, 5, 10, 15, 20\}$ in Reg-GAN's setting and among $\{0.05, 0.1, 0.2, 0.5\}$ in SN-GAN's setting.
We use the Inception Score~(IS)~\citep{salimans2016improved} to evaluate the image quality on CIFAR10 and FID score~\citep{gulrajani2017improved} on both CIFAR10 and CelebA. 
More details about the experimental setting and further results on a synthetic dataset can be found in Appendix~E.

We compare with two typical families of GANs. The first one is referred as unregularized GANs, including WGAN~\cite{arjovsky2017wasserstein}, SGAN~\cite{goodfellow2014generative}, LSGAN~\cite{mao2017least} and Hinge-GAN~\cite{miyato2018spectral}.The second one is referred as reguarlized GANs, including Reg-GAN~\cite{mescheder2018training} and SN-GAN~\cite{miyato2018spectral}.
We emphasize that the regularzied GANs are the previous SOTA methods and our implementations are based on the officially released code. 
For clarity, we refer to our method as CLC-GAN$(\cdot)$ with the hyperparameter $\lambda$ denoted in the parentheses.

In the following, we will demonstrate that (1) the CLC can stabilize GANs using less computational cost than competitive regularizations and is applicable to various objective functions ; (2) CLC-GAN provides a consistent improvement on the quantitative results in different settings compared to related work~\cite{mescheder2018training} and surpasses previous state-of-the-art~(SOTA) GANs~\cite{miyato2018spectral,zhang2019consistency}.

\begin{table}
	\centering
	\caption{The FID Score on CIFAR10. The results reported here are the best results over the training process and are averaged over 3 runs.}
	\begin{tabular}{ccc}
		\toprule
		Method   & \multicolumn{1}{c}{WGAN} & \multicolumn{1}{c}{SGAN} \\
		\midrule
		No Regularization & $105.21$ & $28.51$ \\
		Reg-GAN & $30.43$ & $28.39$ \\
		Gradient Penalty & $28.20$ & $-$ \\
		\midrule
		CLC-GAN(2)& $23.53\pm 1.22$ & $21.63\pm 0.47$ \\
		CLC-GAN(5)& $21.46\pm 1.57$ & $\mathbf{21.52}\pm 0.96$ \\
		CLC-GAN(10)& $\mathbf{21.14}\pm 1.84$ & $22.20\pm 2.07$ \\
		\bottomrule
	\end{tabular}%
	\label{tab:FID_results}
\end{table} 

\setlength{\tabcolsep}{1.mm}{
\begin{table}[t]
	\centering
	\caption{The Inception score on CIFAR10. $\dag$~\cite{yang2017lr}, $\ddag$~\cite{miyato2018spectral}, $\S$~\cite{zhang2019consistency}. Results of CLC-GAN are averages over 3 runs. }
	\begin{tabular}{cccc}
		\toprule
		Method  & WGAN  & SGAN & Hinge \\
		\midrule
		LR-GAN$^\dag$ & - & $7.17$ & - \\
		SN-GAN$^\ddag$ & - & - & $8.22$\\
		CR-GAN$^\S$ & - & $8.40$ & - \\
		\midrule
		Gradient Penalty & $7.82$  & - & - \\
		Reg-GAN   & $7.34$  & $7.37$  &  $7.37$\\
		CLC-GAN(2) & $8.42\pm .06$  & $8.28\pm .05$ & $8.49\pm .08$ \\
		CLC-GAN(5) & $\mathbf{8.49}\pm .07$  & $8.44\pm .08$ & $\mathbf{8.54}\pm .03$\\
		CLC-GAN(10) & $8.38\pm .10$  & $\mathbf{8.47}\pm .09$ & $8.46\pm .00$ \\
		\midrule
		SN-GAN & $3.29$& $8.17$ & $8.28$ \\
		\tabincell{c}{CLC-SN-\\GAN(0.1)} & $8.14\pm .02$ & $8.30\pm .09$ & $\mathbf{8.54} \pm .03$ \\
		\bottomrule
	\end{tabular}%
	\label{tab:is_cifar10}
\end{table}}

\subsection{CLC-GAN is stable}
\label{sec:stability_improvement}

\begin{figure}
	\centering
	\includegraphics[width=0.2\textwidth]{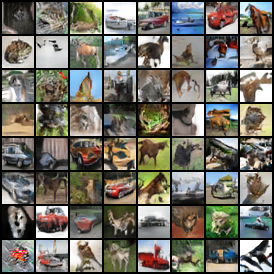}
	\includegraphics[width=0.2\textwidth]{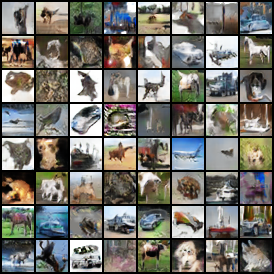}
	\includegraphics[width=0.2\textwidth]{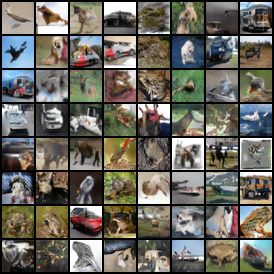}
	\includegraphics[width=0.2\textwidth]{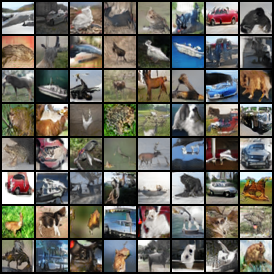}
	\caption{The generated results of CIFAR10 dataset. From top left to bottom right: WGAN-GP, Reg-WGAN, CLC-WGAN(5), CLC-SGAN(5).}
	\label{fig:cifar_generation}
\end{figure}

\begin{figure}[t]
	\centering
	\includegraphics[width=0.2\textwidth]{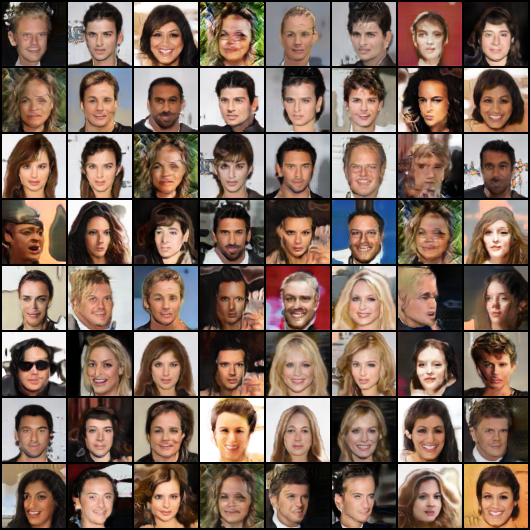}
	\includegraphics[width=0.2\textwidth]{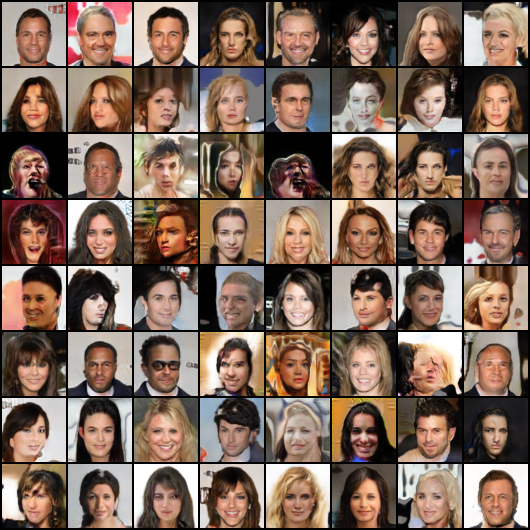}
	\includegraphics[width=0.2\textwidth]{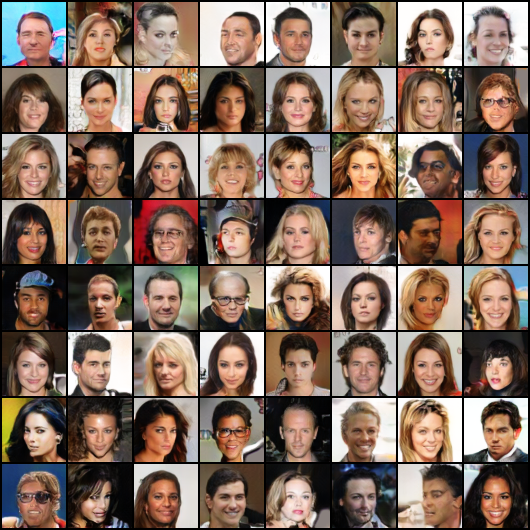}
	\includegraphics[width=0.2\textwidth]{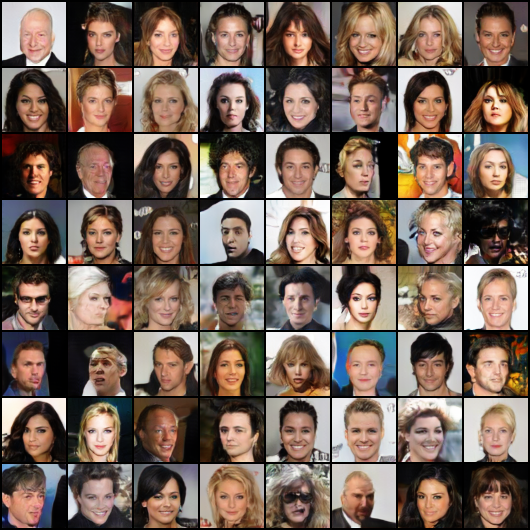}
	\caption{The generated results of CelebA dataset. From top left to bottom right: WGAN-GP, Reg-WGAN, CLC-WGAN(15), CLC-SGAN(15).}
	\label{fig:celeba_generation}
\end{figure}

In the linear case, the simulated results in \fig{simulated_diracgan} demonstrate that CLC-GAN can stabilize the Dirac GAN, which agrees with our theoretical analysis in \secref{ana_imp_stability}. 

In normal GANs, we compare CLC-GAN with a wide range of GANs~\cite{arjovsky2017wasserstein,goodfellow2014generative,mao2017least,miyato2018spectral} and their regularized version in \cite{mescheder2018training} in terms of training stability qualitatively.
The learning curves are shown in \fig{learning_curve}. The top panel shows the IS on CIFAR10 and the bottom one shows FID on CelebA.

In both panels, the training dynamics of unregularized GANs are not stable. On CIFAR10, the unregularized GANs all diverge from the data distribution and on CelebA they even diverge at the very beginning.
Indeed, their FID results on CelebA are over $300$ which is too large to be shown in the figure. 
Among unregularized GANs, LSGAN and SGAN are more stable than WGAN on CIFAR10 which is consistent to our analysis in \tabl{summarize_dynamic}.
However, none of them provide converged results, nor can they generalize to larger images in CelebA.
We hypothesize that the nonlinearity in neural networks is the main reason for the divergence behaviour.
Instead, CLC-GAN can succetssfully avoid the oscillatory behaviour and regularize GANs towards the data distribution. 
The robustness of CLC-GAN in the nonlinear dynamics agrees with the theoretical analysis in \tabl{summarize_dynamic} and the experience in control theory, which are the main motivations of our paper.
In conclusion, the comparison between the unregularized GANs and their controlled versions show the effectiveness of the proposed method.

Indeed, Reg-GAN can also stabilize the training dynamics.
In comparison, the CLC-GANs are computationally efficient and achieve better results after convergence.
First, unlike the gradient penalty which implies a non-trivial running time~\cite{kurach2018large}, CLC-GANs directly regularize the activation of $\xD$ and require less computational cost.
For instance, our method can conduct approximate $8$ iterations per second of training on CelebA whereas Reg-GAN can only conduct $4$ iterations per second on Geforce 1080Ti. 
Second, CLC-GANs provide higher IS on CIFAR10 and lower FID on CelebA as qualitatively shown in the learning curves. The quantitative results are summarized in the following subsection.

\fig{cifar_generation}~\&~\fig {celeba_generation} show the generated samples. Those from CLC-GAN are semantically meaningful in all setting and are at least competitive to the ones from very strong baselines.

\subsection{Quantitative Results}
\label{sec:quantative_results}

We now present the quantitative results on CIFAR10 in the settings that include different objective functions, neural network architectures and the values of $\lambda$. The IS and FID are shown in \tabl{is_cifar10} and \tabl{FID_results} respectively. The comparisons among different settings are given within the tables.


First, our method provides a consistent improvements on both IS and FID on CIFAR10. For FID, CLC-GANs decrease it from $28$ to $23$ compared to Reg-GAN. 
For IS, CLC-GANs surpass previous SOTA GANs. Specifically, CLC-GANs achieve IS over $8.45$ with various objectives without using spectral normalization, which is a significant improvement compare to related works, including SN-GAN~\cite{miyato2018spectral} and CR-GAN~\cite{zhang2019consistency}. 

Second, CLC-GAN is also applicable to SN-GAN's architecture and improve its performance, whereas most gradient-based regularizations fail to introduce significant improvement~\cite{kurach2018large}.
Unlike SN-GAN whose performance largely depends on the objective functions, CLC-SN-GAN provides stable training dynamics consistently. 

Finally, CLC-GAN is not very sensitive to the hyperparameter $\lambda$ given the normalization used in $\xD$. When batch normalization is adopted, CLC-GANs with $\lambda=2, 5, 10$ all achieve SOTA IS and a large improvement on FID. When spectral normalization~\cite{miyato2018spectral} is used, a relatively smaller $\lambda$ is required. 
Besides the reported results with $\lambda=0.1$, CLC-SN-GANs with $\lambda\in\{0.05, 0.2\}$ achieves IS over $8.4$ consistently using Hinge loss.
The underlying mechanism of the difference between the two types of normalizations is unclear. We hypothesize that it is because $\xD$ is a Lipschitz-1 function with spectral normalization.

\section{Conclusions and Discussions}

In this paper, we propose a novel perspective to understand the dynamics of GANs and a stabilizing method called CLC-GAN. We model the dynamics of the Dirac GAN with linear objectives theoretically in the frequency domain and extend the analysis to nonlinear objectives using local linearization.
By leveraging the recipe for control theory, we propose a stabilizing method called CLC to improve Dirac GAN's stability and generalize CLC to normal GANs. The simulated results on Dirac GAN and empirical results on normal GANs demonstrate that our method can stabilize a wide range of GANs and provide better convergence results.

Although CLC-GAN provides promising results, further analyses can be done to achieve better results. On one hand, our analysis mainly focuses on the continuous cases, where the practical implementation optimizes both $\xG$ and $\xD$ in discrete time steps. In this case, the $Z$-transform is a better tool than LT used in this paper. On the other hand, we approximate the dynamics in the function space using the update in the parameter space, which can be improved by recent analyses of GANs in the function space~\citep{johnson2018composite}.
Finally, modern control theory and non-linear control methods~\citep{khalil2002nonlinear} can potentially help GANs to achieve better performance. These are promising directions for the future work.

\section*{Acknowledgements}
This work was supported by the National Key Research and Development Program of China (No. 2017YFA0700904), NSFC Projects (Nos. 61620106010, U19B2034, U181146), Beijing NSF Project (No. L172037), Beijing Academy of Artificial Intelligence, Tsinghua-Huawei Joint Research Program, Tiangong Institute for Intelligent Computing, and NVIDIA NVAIL Program with GPU/DGX Acceleration. C. Li was supported by the Chinese postdoctoral innovative talent support program and Shuimu Tsinghua Scholar.

\bibliographystyle{icml2020}
\bibliography{references}

\appendix

\section{Dynamics for different GANs.}

In this section, we apply the local linearization technique to Dirac GANs with various objective functions, including vanilla GAN, non-saturation GAN~\cite{goodfellow2014generative}, LS-GAN~\cite{mao2017least} and Hinge-GAN~\cite{miyato2018spectral}. Following the notations in the main body, the training dynamics of general Dirac GANs are given by:
\begin{align}
    \frac{d\tphi(t)}{dt} &= \frac{\partial \uV_1(\phi;\theta)}{\partial \phi}|_{\phi=\tphi(t), \theta=\ttheta(t)} \\
    & = h_1^\prime(\tphi(t) c)c - h_2^\prime(\tphi(t)\ttheta(t))\ttheta(t), \\
    \frac{d\ttheta(t)}{dt} &= \frac{\partial \uV_2(\theta;\phi)}{\partial \theta}|_{\phi=\tphi(t), \theta=\ttheta(t)} \\
     &= h_3^\prime(\tphi(t)\ttheta(t))\tphi(t).
\end{align}
By applying the local linearization technique to both $\phi$ and $\theta$ around the equilibrium point $(\tphi_c, \ttheta_c) = (0, c)$, the dynamic can be approximated as:
\begin{align}
    \begin{bmatrix} \frac{d\tphi(t)}{dt} \\ \frac{d\ttheta(t)}{dt} \end{bmatrix} &\approx  \begin{bmatrix}\frac{\partial^2 \uV_1(\phi;\theta)}{\partial\phi^2} & \frac{\partial^2 \uV_1(\phi;\theta)}{\partial \theta\partial \phi} \\
    \frac{\partial^2 \uV_2(\phi;\theta)}{\partial \phi\partial\theta} & \frac{\partial^2 \uV_2(\phi;\theta)}{\partial\theta^2}
    \end{bmatrix}\begin{bmatrix}\tphi(t) - \tphi_e \\ \ttheta(t) - \ttheta_e\end{bmatrix} \\
    &= T \begin{bmatrix}\tphi(t) - \tphi_e \\ \ttheta(t) - \ttheta_e\end{bmatrix} = T \begin{bmatrix}\tphi(t)\\ \ttheta(t) - c\end{bmatrix},
\end{align}
and $T$ can be denoted as:
\begin{align}
    T = \begin{bmatrix}h_1^{\prime\prime}(\phi c)c^2 + h_2^{\prime\prime}(\phi\theta)\theta^2 & h_2^\prime(\phi\theta) + h_2^{\prime\prime}(\phi\theta)\theta\phi \\
    h_3^\prime(\phi\theta) + h_3^{\prime\prime}(\phi\theta)\phi\theta & h_3^{\prime\prime}(\phi\theta)\phi^2
    \end{bmatrix}.\nonumber
\end{align}
Here $h_i^{\prime\prime}(x)$ is the second order derivative of $h_i(x)$ for $i\in\{1,2,3\}$.
Below we assume $c=1$ and start the case by case analysis for various types of GANs.

\subsection{Vanilla GAN}
In vanilla GAN, we have:
\begin{align}
    h_1(x) = \log(\sigma(x)), \\
    h_2(x) = \log(1-\sigma(x)), \\
    h_3(x) = -\log(1-\sigma(x)),
\end{align}
where $\sigma(\cdot)$ denotes the sigmoid function.
Then we have:
\begin{align}
    &h_1^\prime(x) = (1 - \sigma(x)), h_1^{\prime\prime}(x) = -\sigma(x)(1-\sigma(x)),\\
    &h_2^\prime(x) = -\sigma(x), h_2^{\prime\prime}(x) = -\sigma(x)(1-\sigma(x)), \\
    &h_3^\prime(x) = \sigma(x), h_3^{\prime\prime}(x) = \sigma(x)(1-\sigma(x)).
\end{align}
and for $T$:
\begin{align}
    T = \begin{bmatrix}-\frac{1}{2} & -\frac{1}{2} \\ \frac{1}{2} & 0 \end{bmatrix}.
\end{align}
It indicates that 
\begin{align}
    \fPhi(s) = -\frac{1}{2s+1} (\fTheta(s) - \fC(s)) \\
    \fTheta(s) = \frac{1}{2s}\fPhi(s).
\end{align}
Then we can solve the dynamics of vanilla GAN as:
\begin{align}
\begin{cases}
\fPhi(s) = \frac{2s}{4s^2 + 2s + 1}\fC(s), \\
\fTheta(s) = \frac{1}{4s^2 + 2s + 1}\fC(s).
\end{cases}
\end{align}

\subsection{Non-saturation GAN}
Non-saturation GAN~(NS-GAN) shares the same equilibrium point and the objective function for the discriminator. It modifies $h_3$ as $h_3(x) = \log(\sigma(x))$ and we have:
\begin{align}
   h_3^\prime(x) = (1 - \sigma(x)), h_3^{\prime\prime}(x) = -\sigma(x)(1-\sigma(x)).
\end{align}
By substituting the above equation to $T$, the dynamic of NS-GAN is equivalent to vanilla GAN and therefore shares the same transfer function.

\subsection{Hinge GAN}
For Hinge GAN, we have:
\begin{align}
    h_1(x) = \min\{-1+x, 0\}, \\
    h_2(x) = \min\{-1-x, 0\}, \\
    h_3(x) = x.
\end{align}
Then we have:
\begin{align}
    h_1^\prime(x) = 1, h_1^{\prime\prime}(x) = 0, \\
    h_2^\prime(x) = -1, h_2^{\prime\prime}(x) = 0, \\
    h_3^\prime(x) = 1. h_3^{\prime\prime}(x) = 0.
\end{align}
Therefore the Hinge GAN actually shares the same dynamics as WGAN around the equilibrium point.

\subsection{Least Square GAN}

The objective function of least square GAN~(LS-GAN) is:
\begin{align}
    \uV_1(\phi; \theta) = - (\phi c - 1)^2 - (\phi\theta)^2, \\
    \uV_2(\theta; \phi) = -(\phi\theta)^2.
\end{align}
In this case, there's no equilibrium point. We modify the discriminator as $D(x) = \phi x + 0.5$ which is equilvalent to convert the objective functions as follows: 
\begin{align}
    h_1(x) = -(x - 0.5)^2, \\
    h_2(x) = -(x + 0.5)^2, \\
    h_3(x) = -(x-0.5)^2,
\end{align}
and therefore we have:
\begin{align}
    &h_1^\prime(x) = -2(x - 0.5), h_1^{\prime\prime}(x) = -2,\\
    &h_2^\prime(x) = -2(x + 0.5), h_2^{\prime\prime}(x) = -2, \\
    &h_3^\prime(x) = -2(x - 0.5), h_3^{\prime\prime}(x) = -2.
\end{align}
Then the $T$ can be denoted as:
\begin{align}
    T = \begin{bmatrix}-4 & -1 \\ 1 & 0 \end{bmatrix}.
\end{align}
We have:
\begin{align}
    &\fPhi(s) = -\frac{1}{s+4}(\fTheta(s) - \fC(s)), \\
    &\fTheta(s) = \frac{1}{s}\fPhi(s).
\end{align}
Then we can solve the dynamics of LSGAN as:
\begin{align}
\begin{cases}
\fPhi(s) = \frac{s}{s^2 + 4s + 1}\fC(s), \\
\fTheta(s) = \frac{1}{s^2 + 4s + 1}\fC(s).
\end{cases}
\end{align}

\section{Dynamics with Lipschitz Continuity}

In this section, we prove that around the equilibrium, the dynamics of regularized $\xD$ with Lipschitz constraint is equivalent to the unregularized $\xD$ as in Eqn. (20). 
With the dynamics defined by the corresponding gradient flow, we only need to prove that updating $\xD$ according to Eqn. (20) will not violate the Lipschitz constraints, at least locally around the equilibrium. Here we make the following assumptions:
\begin{enumerate}
    \item Both $p(x)$ and $p_\xG(t, x)$ are $C^1$-smooth: $\frac{dp(x)}{dx}$ and $\frac{dp_\xG(t, x)}{dx}$ exists and is continuous $\forall~t$.
    \item $q(x)\to 0$ and $\frac{dq(x)}{dx}\to 0$ when $x\to 0$ for $q\in\{p, p_\xG\}$.
    \item There exists an $M$ such that $|\frac{dq(x)}{dx}|_2 < M$ for $q\in\{p, p_\xG\}$.
\end{enumerate}
The above assumptions are satisfied for most probability density functions.

The distance in the function space is defined as $d(p_1, p_2) = \sup_{x\in\Real^n} |p_1(x) - p_2(x)|$ which always exists because of the 2-nd conditions above. We define $\Omega_{L}=\{p(x)|p(x)\in C^1, |\frac{dp(x)}{dx}|_2 < L~\forall x.\}$ and $B(\epsilon)=\{p(x)|p(x)\in C^1, \sup_{x} |p(x)| < \epsilon\}$.
Then we have the follow proposition:

\begin{proposition}
There exists $\eta>0$, such that $\forall \xD(x)\in \Omega_{0.5}$, we have $\xD(x)+\eta(p(x) - p_\xG(x)) \in \Omega_{1}$. 
\end{proposition}

\begin{proof}
By denoting $\xD^\prime(x) = \xD(x)+\eta(p(x) - p_\xG(x))$, We have:
\begin{align}
    \frac{d(\xD(x)+\eta(p(x) - p_\xG(x)))}{dx} \\
    = \frac{d\xD(x)}{dx} + \eta(\frac{p(x)}{dx} - \frac{p_\xG(x)}{dx}).\nonumber
\end{align}
Therefore, we have 
\begin{align}
    &|\frac{d(\xD(x)+\eta(p(x) - p_\xG(x)))}{dx}|_2 \\
    \leq &|\frac{dD(x)}{dx}|_2 + \eta(|\frac{p(x)}{dx}|_2 + |\frac{p_\xG(x)}{dx})|_2 \\
    \leq &0.5 + \eta(M+M).
\end{align}
By letting $\eta=\frac{1}{4M}$, we have $|\frac{d(\xD^\prime)}{dx}|_2\leq 0.75$. Therefore we have $\xD^\prime(x)\in \Omega_{1}$.
\end{proof}

The above proposition indicates that when $\xD(x)$ is sufficient close to the equilibrium and the learning rate is sufficient small, then the dynamics of $\xD$ still follows Eqn.~(11) for Dirac GAN and Eqn.~(22) for normal GANs. The simulated results of Dirac GAN in Fig.~1 and the bad performance of SN-GAN with WGAN's objective in Sec.~6.2 agree with this argument.

\section{Theoretical Analysis of CLC-GAN}

\subsection{Proof of Lemma 1}

Under the non-parametric setting following~\citet{goodfellow2014generative}, the equilibrium of GAN's minimax problem is achieved when $p_\xG=p$ and $\xD(x) = 0$ for all $x$. Besides, for the regularization term introduced by CLC-GAN:
\begin{align}
    \uR_t(\xD) &= \int_{x\in\mathcal{X}} p_u^t(x)\xD^2(x)dx,
    \label{eqn:clc-reg}
\end{align}
it also achieves optimum when $\xD(x) = 0$ for all $x$. Therefore, regularizing $\xD$'s training dynamic with \eqn{clc-reg} will not change the equilibrium of GANs. This argument for other variants of GANs remains the same under the condition that the equilibrium of unregularized GANs' minimax problem is achieved when $\xD(x) = 0$ for all $x$ around the data distribution. This assumption is meet by most variants of GANs~\cite{arjovsky2017wasserstein,mao2017least,miyato2018spectral,zhang2019consistency}.

\subsection{Proof of Theorem 1}

In this subsection, we provide the proof of Theorem 1, whose proof procedure mainly follows~\citet{mescheder2018training}. We first denote that $\xG_\theta$ is the generator parameterized by $\theta\in \Real^{|\theta|}$ and $\xD_{\phi}$ is the discriminator parameterized by $\phi\in \Real^{|\phi|}$.
Before going to the proof of Theorem 1, we first provide the assumptions we made, which are similar to~\cite{mescheder2018training}.

We first assume the data distribution can be captured by the generator $\xG$, which is identical to the Assumption I in~\cite{mescheder2018training} as:
\begin{assumption}
There exists a discriminator $\xD^*$ parameterized by $\phi^*$ and a generator $\xG^*$ parameterized by $\theta^*$, such that $p_{\xG^*}(x) = p(x)$ for all $x$ and $\xD^*(x) = 0$ in some local neighbourhood of the support of the data distribution $p(x)$.
\label{ass1}
\end{assumption}

Besides, we define 
\begin{align}
    h(\phi) = \ep_{p(x)}[\xD_\phi^2(x)],
\end{align}
Then we define a manifold over the parameter space of $\phi$ and $\theta$ as follows:
\begin{align}
    &\mathcal{M}_\theta = \{\theta| p_{\xG_{\theta}}(x) = p(x),~\forall x\},\\
    &\mathcal{M}_\phi = \{\phi| h(\phi) = 0\}.
\end{align}
Here we use $\xG_\theta$ to denote the generator parameterized by $\theta$. To state the second assumption, we need
\begin{align}
    g(\theta) = \ep_{p_{\xG_\theta}(x)}[\partial_\phi \xD_\phi(x)|_{\phi=\phi^*}].
\end{align}
Then we provide the second assumption by follow~\citet{mescheder2018training} as follows:
\begin{assumption}
There are $\epsilon$-balls around $B_\epsilon(\phi^*)$ and $B_\epsilon(\theta^*)$ around $\phi^*$ and $\theta^*$, such that $\mathcal{M}_\phi\cap B_\epsilon(\phi^*)$ and $\mathcal{M}_\theta\cap B_\epsilon(\theta^*)$ define $C^1$-manifolds. Moreover, the following conditions hold:
\begin{itemize}
    \item if $v\in \Real^{|\phi|}$ is not in the tangent space of $\mathcal{M}_\phi$ at $\phi^*$, then we have $\partial^2_v h(\phi^*) \neq 0$.
    \item if $w\in \Real^{|\theta|}$ is not in the tangent space of $\mathcal{M}_\theta$ at $\theta^*$, then we have $\partial_w g(\theta^*) \neq 0$.
\end{itemize}
\label{ass2}
\end{assumption}

The validness of GAN's training dynamics requires the following assumption:
\begin{assumption}
	The functions $h_1$, $h_2$, $h_3$ requires the following conditions:
	\begin{itemize}
		\item $h_1^\prime(0) > 0$, $h_2^\prime(0) < 0$ and $h_3^\prime(0) > 0$.
		\item $|h_1^\prime(0)| = |h_2^\prime(0)| = |h_3^\prime(0)|$.
	\end{itemize}
	Here $|\cdot|$ denotes the absolute value.
	\label{ass3}
\end{assumption}

The formal statement of Theorem 1 is given as follows:
\setcounter{theorem}{0}
\begin{theorem}
	Assume the assumptions \ref{ass1}, \ref{ass2} and \ref{ass3} hold for $\phi^*$ and $\theta^*$. For small enough learning rate and $\lambda>-h_1^{\prime\prime}(0)-h_2^{\prime\prime}(0), \lambda>0$, training GANs with objectives formulated in Eqn. (19) and the regularization term in Eqn. (25) ensures locally convergence with alternative gradient descent.
\end{theorem}

\begin{proof}
	The gradient flow defined by GAN's objective function is given as:
	\begin{align}
	v(\phi, \theta) = \begin{pmatrix}
	\nabla_\phi \uV_1(\xD; \xG) \\
	\nabla_\theta \uV_2(\xG; \xD)
	\end{pmatrix}.
	\end{align}
	Here we have:
	\begin{align}
	\nabla_\phi \uV_1(\xD; \xG) = \ep_{p(x)}[h_1^\prime(\xD(x))\frac{\partial \xD(x)}{\partial \phi}] \\ + \ep_{p_\xG(x)}[h_2^\prime(\xD(x))\frac{\partial \xD(x)}{\partial \phi}],
	\end{align}
	and 
	\begin{align}
	\nabla_\theta \uV_2(\xG; \xD) = \ep_{p_z(z)}h_3^\prime(\xD(\xG(z))))\frac{\partial \xD(\xG(z))}{\partial \xG(z)} \frac{\partial \xG(z)}{\partial \theta}.
	\end{align}
	
	Then the Jacobian matrix of the gradient flow is:
	\begin{align}
	J_U(\phi, \theta) &= \begin{pmatrix}
	\nabla^2_{\phi} \uV_1(\xD; \xG) & \nabla^2_{\theta\phi} \uV_1(\xD; \xG) \\
	\nabla^2_{\phi\theta} \uV_2(\xG; D) & \nabla^2_{\theta} \uV_2(\xG; \xD)
	\end{pmatrix} \\
	& = \begin{pmatrix}
	J_{DD}(\phi, \theta) & J_{GD}(\phi, \theta) \\
	J_{DG}(\phi, \theta) & J_{GG}(\phi, \theta)
	\end{pmatrix}.
	\end{align}	
	Note that at the equilibrium point $(\phi^*, \theta^*)$, we have $\xD^*(x) = 0$ around the support the data distribution. Therefore, we have $\frac{\partial \xD(x)}{\partial x} = 0$ and $\frac{\partial^2 \xD(x)}{\partial x^2} = 0$ for $x\sim p(x)$. It is easy to verify that $J_{GG}(\phi^*, \theta^*)$, i.e., $\nabla^2_{\theta} \uV_2(\xG^*; \xD^*)$, is a zero matrix. Similar, we have:
	\begin{align}
	J_{GD}(\phi^*, \theta^*) &= \nabla_\theta (\ep_{p_\xG(x)}[h_2^\prime(\xD(x))\frac{\partial \xD(x)}{\partial \phi}])\nonumber \\
	&=   \nabla_\theta (\ep_{p_z(z)}[h_2^\prime(\xD(\xG(z)))\frac{\partial \xD(\xG(z))}{\partial \phi}])\nonumber \\
	& = \ep_{p_z(z)}[h_2^\prime(0)\frac{\partial^2 \xD(\xG(z))}{\partial \theta \partial \phi}] \nonumber\\
	& = \ep_{p_z(z)}[h_2^\prime(0)\frac{\partial^2 \xD(\xG(z))}{\partial \xG(z) \partial \phi}\frac{\partial \xG(z)}{\partial \theta}]
	\end{align}
	and
	\begin{align}
	J_{DG}(\phi^*, \theta^*) &= \nabla_\phi  \ep_{p_z(z)}h_3^\prime(\xD(\xG(z))))\frac{\partial \xD(G(z))}{\partial \xG(z)} \frac{\partial \xG(z)}{\partial \theta} \nonumber\\
	&= 	 \ep_{p_z(z)}h_3^\prime(\xD(\xG(z))))\frac{\partial^2 \xD(\xG(z))}{\partial \phi \partial \xG(z)} \frac{\partial \xG(z)}{\partial \theta} \nonumber \\
	&= 	 \ep_{p_z(z)}h_3^\prime(0)\frac{\partial^2 \xD(\xG(z))}{\partial \phi \partial \xG(z)} \frac{\partial \xG(z)}{\partial \theta}
	\end{align}
	Since $h_2^\prime(0) = - h_3^\prime(0)$, we have $J_{DG} = - J_{GD}^T$.
		
	Note that with sufficient small learning rate, we have $p_u^t(x) = p_\xG(x) = p(x)$ around the equilibrium. Then we provide the gradient flow and it's Jacobian matrix of the regularization term 
	\begin{align}
	\uV_R(\phi, \theta) = \begin{pmatrix}
	\nabla_\phi -\uR_t(\xD) \\
	\nabla_\theta -\uR_t(\xD)
	\end{pmatrix}.
	\end{align}
	Since $\uR_t(D)$ is simply a function of $\phi$, $\nabla_\theta \uR_t(D)$ is a zero vector. The Jacobian matrix of the regularization term is given as:
	\begin{align}
	J_R(\phi, \theta) = \begin{pmatrix}
	\nabla^2_{\phi} -\uR_t(D)  & 0 \\
	0 & 0
	\end{pmatrix}.
	\end{align}
	
	With the Jacobian matrix of the regularized dynamics formulated as:
	\begin{align}
	J = J_U + J_R = \begin{pmatrix}
	J_{DD} - \nabla^2_{\phi} R_t(D) & J_{DG} \\
	J_{GD} & 0
	\end{pmatrix},
	\end{align}
	We can directly follow the proof of Theorem 4.1 in~\citet{mescheder2018training} in Appendix D.
\end{proof}

\subsection{Interpreting CLC-GAN in the parameter space}
In this paper, we mainly analyze our proposed method in the function space, including dynamic analysis and controller designing. Instead, our proposed method can also be interpreted as certain regularization terms on the Jacobian matrix of the training dynamics. Below we provide a formal demonstration.

First, we denote the equilibrium of $\xG$ and $\xD$ as $(\theta^*, \phi^*)$, where $p_\xG(x;\theta^*) = p(x)$ and $\xD(x; \phi^*) = 0$ for all $x$. Note that $\phi^*$ is also a global minimum point of the regularization term $\uR(D)=\int \xD^2(x) dx$. Then we have $\frac{\partial^2 \uR(\xD)}{\partial\phi^2}\succeq 0$.

We denote $U(\xD, \xG)$ as the objective function of the minimax optimization problem in WGAN without CLC regularization. Then the Jacobian matrix of the training dynamic can be denoted as:
\begin{align}
    J = \begin{pmatrix} \frac{\partial^2 \uU(\xD, \xG)}{\partial \phi^2} & \frac{\partial^2 \uU(\xD, \xG)}{\partial\phi\partial\theta} \\
    \frac{\partial^2 \uU(\xD, \xG)}{\partial\theta\partial\phi} & \frac{\partial^2 \uU(\xD, \xG)}{\partial \theta^2}
    \end{pmatrix}.
\end{align}
Because of the linearity of the derivation operation, the training dynamics of the WGAN with CLC regularization is denoted as:
\begin{align}
    J^\prime = J-J_L = J - \begin{pmatrix} \frac{\partial^2 L(\xD)}{\partial \phi^2} & \mathbf{0} \\
    \mathbf{0} & \mathbf{0}
    \end{pmatrix},
\end{align}
where we abuse the $\mathbf{0}$ to denote the zero matrix with certain size to match the size of $J$.
Since $\frac{\partial^2 \uR(\xD)}{\partial\phi^2}\succeq 0$, we have $-J_L\preceq 0$. Therefore, the CLC regularization introduces a negative semi-definite matrix to the original Jacobian matrix, which is helpful to stabilize the training dynamics of GANs.

\section{Understanding Existing Work as Closed-loop Control}

A side contribution of this paper is to understand existing methods~\cite{gidel2018negative} uniformly as certain CLC controllers.
The momentum is an example where \citet{gidel2018negative} provide some theoretical analysis of momentum in training GANs. Here we re-analyze the momentum using Dirac GAN under the perspective of control theory. 

The momentum method~\citep{qian1999momentum} is powerful when training neural networks, whose theoretical formulation is given by:
\begin{align}
\tilde{\phi}_{t+1} = \beta \tilde{\phi}_t + (1 - \beta) \nabla \phi_t,~\phi_{t+1} = \phi_t + \eta \tilde{\phi}_{t+1},
\label{eqn:momentum_theory}
\end{align}
where $\nabla\phi$ is the input of $\phi$'s dynamic, i.e., $u_\xD = c-\theta$.
The $\beta$ is the coefficient for the exponential decay.
However, momentum instead is not helpful when training GANs~\cite{radford2015unsupervised,mescheder2018training,brock2018large,gulrajani2017improved} where smaller $\beta$ or even zero is recommended to achieve better performance.

In control theory, the momentum is equivalent to adding an exponential decay to the input of the dynamics~\citep{an2018pid}:
\begin{align}
\tilde{\tth}(t) = \int_0^t \tth(u) \exp(-\tau(t-u))du.
\end{align}
The LT of an exponential decay dynamic is $\frac{1}{s+\tau}$, i.e., $\tilde{\fH}(s) = \frac{1}{s+\tau}\fH(s)$. $\tau > 0$ denotes the decay coefficient which depends on $\beta$.
Therefore, we can formulate the dynamics of Dirac GAN in the following:
\begin{align}
\begin{cases}
m_{\tphi}(t) = \int_0^t \tc(u) - \ttheta(u)) \exp(-\tau(t-u))du,\\
\frac{d\tphi}{dt} = m_{\tphi}(t), \\
\frac{d\ttheta}{dt} = \tphi(t).
\end{cases}\label{eqn:dynamic_momentum}
\end{align}
By applying LT, we have $\fM_{\tphi}(s) = \frac{1}{s+\tau}(\fC(s)-\fTheta(s))$ and $\fPhi$ can be represented as:
$$\fPhi(s) = \frac{s}{s^3 + \tau s^2 + 1}\fC(s).$$
With a positive $\tau$, there is at least one pole of this dynamic whose real part is larger than $0$, indicating the instability of the dynamics for GANs with momentum.
The result is consistent with~\cite{gidel2018negative}.

\section{Further Experimental Results on Synthetic Data}

\begin{figure*}
    \centering
    \includegraphics[width=\textwidth]{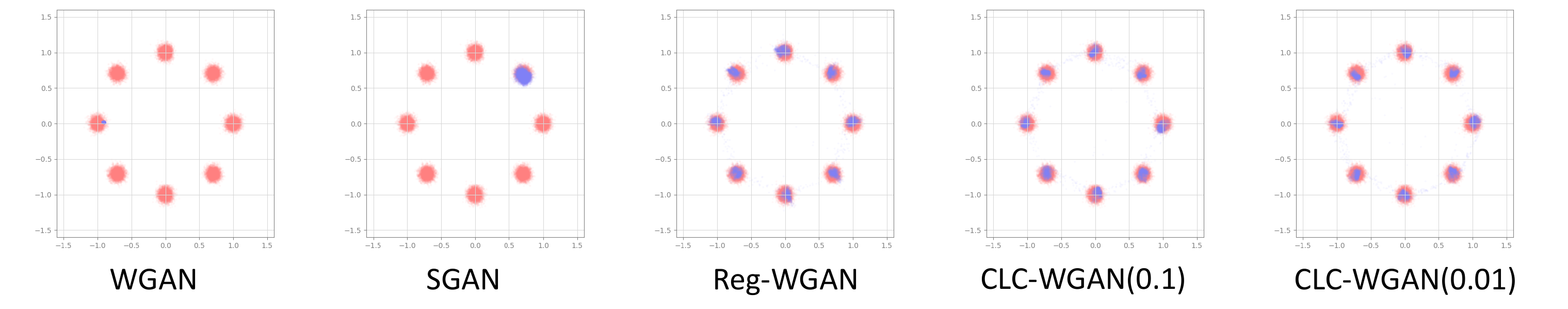}
    \caption{The generated samples for mixture of gaussian distribution. The red points demonstrate the location of data distribution and the blue points are generated samples. Each distribution is plotted using kernel density estimation with 50,000 samples. }
    \label{fig:result_toy}
\end{figure*}

In this section, we evaluate our proposed method on a mixture of Gaussian on the two dimensions. The data distribution consists of $8$ 2D isotropic Gaussian distributions arranged in a ring, where the radius of the ring is $1$, and the deviation of each component Gaussian distribution is $0.05$. For the coefficient $\lambda$, we follow the setting in the spectral normalization as $\lambda\in\{0.01, 0.05, 0.1\}$. We adopt two-layer MLPs for both the generator and the discriminator which consist of $128-512$ units. The batch size is is 512.

The generated results are illustrated in \fig{result_toy} and we further provide the dynamics of the generator distribution in \fig{dynamic_toy}. As we can see, the unregularized WGAN and SGAN suffer from severe model collapse problem and cannot cover the whole data distribution. Besides, the oscillation can be observed during the training process of WGAN: the generator distribution oscillates among the modes of data distribution.
Our method can successfully cover all modes compared to the WGAN and SGAN and the dynamics are converged instead of oscillation.


\begin{figure*}
    \centering
    \includegraphics[width=\textwidth]{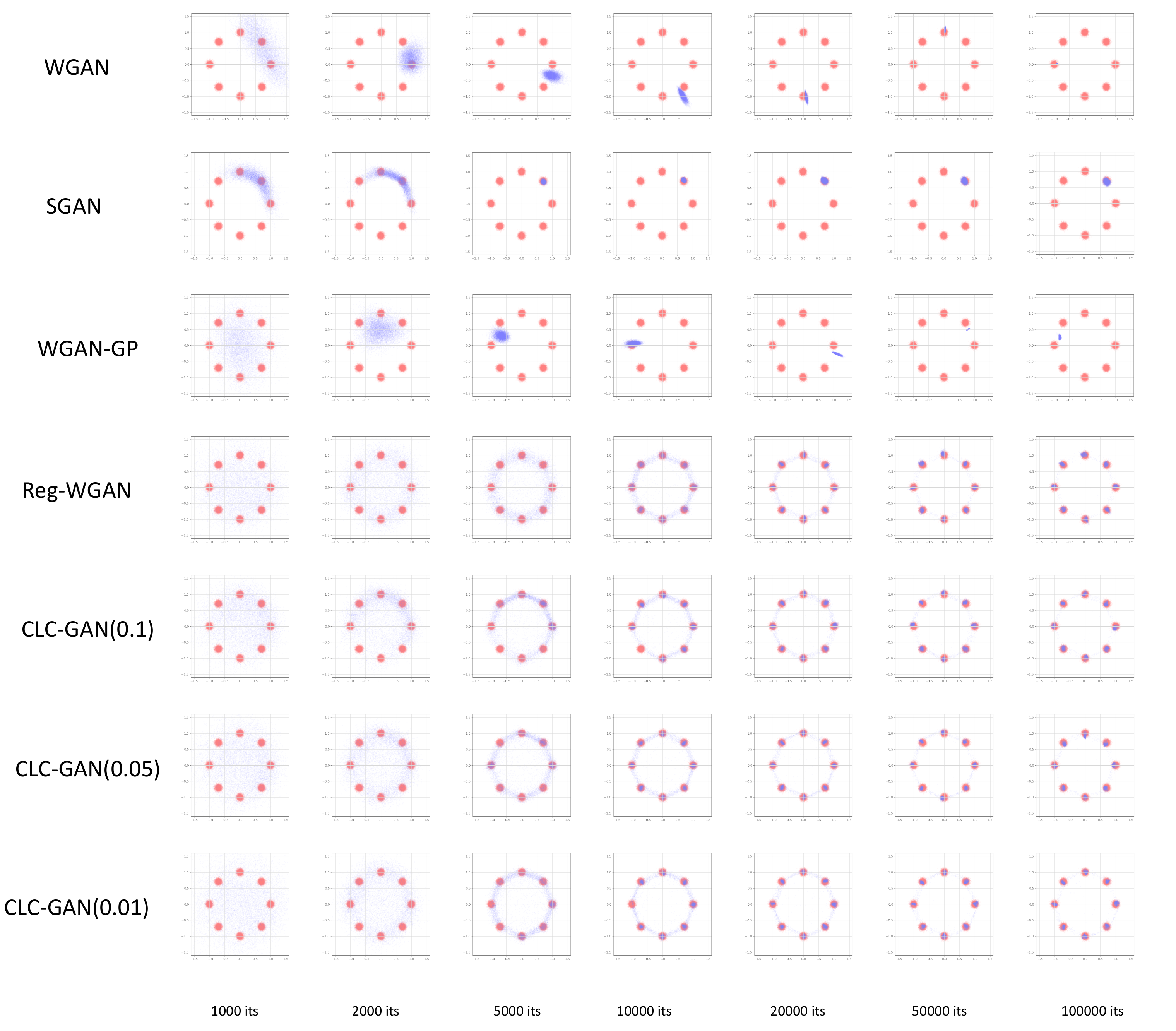}
    \caption{The training dynamics of various GANs on synthetic data.}
    \label{fig:dynamic_toy}
\end{figure*}

\end{document}